\documentclass{article}
\usepackage{arxiv}

\usepackage[utf8]{inputenc} 
\usepackage[T1]{fontenc}    
\usepackage{hyperref}       
\usepackage{url}            
\usepackage{booktabs}       
\usepackage{amsfonts}       
\usepackage{nicefrac}       
\usepackage{microtype}      
\usepackage{lipsum}		
\usepackage{graphicx}
\usepackage{doi}
\usepackage{amsthm}
\usepackage{amsmath}
\usepackage{mathrsfs}%
\usepackage{comment}%
\usepackage[title]{appendix}%
\usepackage{xcolor}%
\usepackage{textcomp}%
\usepackage{manyfoot}%
\usepackage{booktabs}%
\usepackage{algorithm}%
\usepackage{algorithmicx}%
\usepackage{algpseudocode}%
\usepackage{listings}%

\theoremstyle{plain}
\newtheorem{theorem}{Theorem}
\newtheorem{lemma}[theorem]{Lemma}

\newtheorem{rem}{Remark}

\theoremstyle{definition}
\newtheorem{definition}{Definition}

\theoremstyle{remark}

\theoremstyle{prop}
\newtheorem{prop}{Proposition}

\title{Information-Geometric Optimization on Spheres}

\author{Vladimir Ja\' cimovi\' c \\
	Faculty of Natural Sciences and Mathematics\\
	University of Montenegro\\
	Cetinjski put bb., 81000 Podgorica\\
	Montenegro\\
	\texttt{vladimirj@ucg.ac.me} \\
}

\renewcommand{\shorttitle}{IGO on Spheres}

\hypersetup{
pdftitle={Information-Geometric Optimization on Spheres},
pdfsubject={-},
pdfauthor={Vladimir Ja\' cimovi\' c},
}

\begin{document}
\maketitle

\begin{abstract}
We consider the black-box optimization problem on a sphere. Two information-geometric optimization flows (IGO flows) are designed with rigorous calculation of natural search gradients based on hyperbolic (information) geometry of Poincar\' e and Bergman balls. We demonstrate that ensembles of generalized Kuramoto oscillators on spheres compute natural search gradients and realize IGO algorithms on both manifolds. The relationship between natural gradient policies in Bergman balls and quantum decision making is pointed out.
\end{abstract}

\keywords{Poincar\' e balls \and Bergman balls \and Poisson-Szeg\" o kernels \and Kuramoto models on spheres \and quantum decision-making}

\section{Introduction}

The framework of Natural Evolution Strategies (NES) \cite{Wierstra} has been established in the first decade of this century as the result of research efforts to extract principled approaches from various stochastic search heuristics. The blueprint of NES is information-geometric point of view on search strategies \cite{Beyer,OAAH}. Bringing information-geometric concepts into black-box optimization was the key step which enabled transition from evolution strategies \cite{BS} towards more principled NES framework. 

Further conceptual upgrade has been achieved with the introduction of information-geometric optimization (IGO) \cite{OAAH}. IGO presents a mathematically rigorous approach to stochastic search by constructing a continuous-time flow over a smooth parametric family of probability distributions on a search space. Trajectories of this flow follow the natural gradient of a function on the parameter space w. r. to the chosen family. This method can also be interpreted as an infinitesimal maximum likelihood update under the model assumption defined by the family of distributions. This interpretation is particularly transparent for exponential families on Euclidean spaces. Yet another interpretation of IGO flows is maximization of the expectation of the objective function under constrained loss of diversity \cite{Beyer,OAAH}. IGO algorithms are derived from the IGO flow through time discretization and sampling from the current probability distribution at each time step.

IGO emphasizes invariance as the key property for the design of efficient and tractable search strategies. In particular, efficiency of CMA-ES (Covariance Matrix Adaptation - Evolution Strategies), the most successful black-box optimization algorithm on continuous search space, is due to the invariance of the Gaussian family w. r. to linear-affine transformations of the Euclidean space. From this point of view, IGO algorithms can be seen as manifestations of the general rule of geometric probability \cite{Santalo} stating that natural probability models are those which are invariant under actions of a transformation group. In addition, IGO flows are invariant w. r. to reparametrization of probability distributions, as well as w. r. to strictly increasing transformations of the objective function. 

IGO algorithms are derived from the IGO flow through time discretization and sampling from the current probability distribution at each time step.
Theoretically, the NES/IGO framework enables the design of search strategies based on any smooth parametric family of probability distributions. In practice, the Gaussian family (including sub-families with diagonal covariance matrices) is typically used when optimizing over continuous Euclidean spaces. As mentioned above, this choice gives rise to versions of the CMA-ES algorithm \cite{HO,HA,ANOK}. The inherent drawback of this choice is slow exploration of larger regions in the search space due to light tails of Gaussian distributions. This observation motivated studies on alternative NES algorithms with heavy-tailed probability distributions (such as multivariate Cauchy), thus enabling faster evolution towards distant regions in the search space \cite{SGS}. However, this choice is associated with a number of difficulties related to computation of Fisher information and parameter updates. 

In the last decade, there was a convergent evolution between NES and reinforcement learning (RL). This convergence is not surprising in the view of the fact that natural policy gradients introduced in RL by Kakade \cite{Kakade} stem from essentially the same idea which underlies NES. Therefore, relevance of NES and IGO principles extends beyond the narrow black-box optimization context. These principles are incorporated in various RL algorithms or studied as an alternative to RL \cite{SHCSS,KO}. NES are particularly applicable in episodic RL \cite{HN}.

Although IGO presents a universal framework for the black-box optimization over arbitrary search spaces, almost all successful realizations are on finite or Euclidean search spaces. Stochastic search methods over curved manifolds are sparse and underexplored. At the same time, RL problems with curved action spaces, such as \cite{JA} and hyperbolic balls \cite{CCBH} attract an increasing attention. Black-box optimization on spheres, or rotation groups \cite{CJK} became particularly relevant with the advent of RL with deep feature representations, as in many RL models features are represented by unit vectors, that is - by points on the sphere \cite{CJK}. Obviously, such problems naturally arise when dealing with directional (spherical) data, space navigation or robotic motions. However, deep spherical features are ubiquitous in DL and appear in transformers \cite{Meng}, convolutional neural networks \cite{CCG}, neural evolution pathways \cite{SNHB}, etc.  

In the present paper we consider the following optimization problem on a sphere
\begin{equation}
\label{fitness}
\mbox{minimize  } f(y) \; \mbox{ w.r. to} \quad y \in {\mathbb S}^{d-1}.
\end{equation}
where $f(\cdot)$ is a black-box objective function.

We say that \eqref{fitness} is a {\it directional black-box optimization} problem. We will rigorously derive NES for this problem based on two families of probability distributions on spheres. Along the exposition we will unveil that invariance properties of our families underlie efficient and elegant algorithms, strongly related with conformal geometry and Lorentz groups. Thus, our findings are precisely in line with the general principle of IGO in a rich geometric context. 

A brief explanation of natural search gradients over arbitrary search spaces is provided in the next section. In Section 3 we explain the Cauchy family of probability distributions on spheres in real vector spaces and demonstrate that endowing this family with the Fisher information metric yields a statistical manifold isomorphic (up to a constant multiplier) to the Poincar\' e ball. In Section 4 we compute natural search gradients for this family and construct the corresponding IGO flow. We further demonstrate that this flow is generated by the real Kuramoto model on a sphere. We further instantiate this IGO flow into two IGO algorithms which exploit Kuramoto dynamics for estimations of natural search gradients. In Section 5 we introduce another family of probability distributions on spheres in complex vector spaces and demonstrate that the Fisher information metric turns this family into a statistical manifold isomorphic to the Bergman ball. In Section 6 we compute natural search gradients for this family and demonstrate their relationship with automorphisms of the Bergman ball. We further present the complex projective Kuramoto model of generalized oscillators on the sphere which generates this flow. Finally, we propose two IGO algorithms derived from IGO flows in the Bergman balls. Section 7 contains concluding remarks and an outlook for further developments and applications in evolutionary optimization, directional (including quantum) decision-making and RL.

\section{Search gradients in black-box optimization}

Given an optimization problem with the fitness function $f(x)$ over a search space $X$, we can design a search strategy (or search policy) based on a family of probability distribution $\pi_\theta$ parametrized by $\theta \in \Theta$. We will assume that distributions $\pi_\theta$ are absolutely continuous and denote the corresponding densities by $p(x;\theta)$. We can perform the search by sampling from $\pi_\theta$ and updating parameter $\theta$ based on evaluations of $f(\cdot)$.
 
Such search strategy can be written as an optimization problem over the space $\Theta$: 
\begin{equation}
\label{exp_fitness1}
J(\theta) = {\mathbb E}_{x \sim \pi_\theta} [f(x)] = \int f(x) p(x \, | \, \theta) dx \to \max_{\theta}.
\end{equation}  

At the first glance, calculation of gradients of the objective function in \eqref{exp_fitness1} may seem unfeasible. However, the {\it log-likelihood trick} yields
$$
\nabla_\theta J(\theta) = \nabla_\theta \int f(x) p(x \, | \, \theta) dx = \int f(x) \nabla_\theta p(x \, | \, \theta) dx = 
$$
$$
\int f(x) \nabla_\theta p(x \, | \, \theta) \frac{p(x\, | \, \theta)}{p(x \, | \, \theta)} dx =
 \int [f(x) \nabla_\theta \log p(x \, | \, \theta)] p(x\, | \, \theta) dx = 
$$
$$ 
 {\mathbb E}_\theta [f(x) \nabla_\theta \log p(x \, | \, \theta)].
$$
The last expression unveils that search gradients can be estimated from samples $x_1,\dots,x_M$ in the following way
\begin{equation}
\label{grad_est}
\nabla_\theta J(\theta) \approx \frac{1}{M} \sum_{i=1}^M f(x_i) \nabla_\theta \log p(x_i \, | \, \theta).
\end{equation}

Then the gradient descent algorithm may be written as follows:
\begin{equation}
\label{gradient_ascent}
\theta^{(t+\delta t)} = \theta^t + \delta t \nabla J(\theta^t).
\end{equation}

In the limit of infinitesimally small step $\delta t \to 0$, this algorithm yields the following gradient flow ODE
\begin{equation}
\label{gradientODE}
\frac{d \theta}{dt} = \nabla J(\theta) \bigg|_{\theta = \theta(t)}.
\end{equation}

\subsection{Natural search gradients}

Gradient decent algorithm \eqref{gradient_ascent} assumes the Euclidean metric on the parameter space $\Theta$. In most cases, such an assumption is unjustified and may result in decreased reliability, numerical instabilities and premature convergence of algorithms.

Information geometry introduces the Fisher information metric on the space $\Theta$ thus turning it into a Riemannian manifold. The Fisher information matrix is defined as covariance of the score 
$$
F(\theta) = \mathbb{E}_{\pi_\theta} \left[ \nabla_\theta \log p(x;\theta) (\nabla_\theta \log p(x;\theta))^T \right].
$$
Under regularity assumptions on the family $\pi_\theta$, the above expression for $F(\theta)$ is the positive definite matrix, thus inducing the metric on the manifold $\Theta$. 

Fisher information matrix also appears as the Hessian of the Kullback-Leibler divergence between densities $\pi_{\delta \theta}$ and $\pi_\theta$. For rigorous exposition of information geometry we refer to the pioneering book \cite{AN} or more recent comprehensive and advanced book \cite{AJLS}.

Therefore, the gradient of the function $J$ w. r. to the Fisher information metric (so-called natural gradient \cite{AN,AJLS}) reads
\begin{equation}
\label{nat_grad_est}
\tilde \nabla_\theta J(\theta) \equiv F(\theta)^{-1} \nabla_\theta J(\theta) \approx \frac{1}{M} \sum_{i=1}^M f(x_i) \nabla_\theta \log p(x_i \, | \, \theta).
\end{equation}

Substituting \eqref{nat_grad_est} yields a modification of the "vanilla" gradient flow \eqref{gradientODE}
\begin{equation}
\label{nat_gradientODE}
\frac{d \theta}{dt} = \tilde \nabla J(\theta) \bigg|_{\theta = \theta(t)} = F(\theta)^{-1} \nabla J(\theta) \bigg|_{\theta = \theta(t)}.
\end{equation}

The main contribution of the present paper consists in deriving IGO flows for optimization problems on spheres based on two particular families of probability distributions, as well as proposing computational methods.

\section{Conformally natural family of probability distributions on spheres in real vector spaces}

In this Section we introduce the statistical model for directional NES in real vector spaces. We start with the Poincar\' e ball model of hyperbolic geometry. 

Throughout this Section we denote by $x^T y = x_1 y_1 + \cdots + x_d y_d$ the standard scalar product in $\mathbb R^d$ and by $\| x \| = \sqrt{x^T x}$.

\subsection{Poincar\' e balls and their isometry groups}  

\begin{definition}
\label{Poin_ball_def}
Consider the set $\{x \in \mathbb{R}^d \; : \; \| x \|<1\}$ equipped with the metric $$g_x(u,v) = \frac{u^T v}{(1-\| x \|^2)}, u, v\in \mathbb{R}^d.$$ This manifold is named the $d$-dimensional Poincar\' e ball. We will denote it by $\mathbb{B}^d$ throughout this paper.
\end{definition}


Let $x,y \in \mathbb{B}^d$. The Poincar\'e distance on $\mathbb{B}^d$ is given by
\begin{equation}\label{pome}
d_h(x,y)=\frac{1}{2}\log \frac{1+R}{1-R},
\end{equation} 
where
 \begin{equation}
 \label{rho}
 R=\frac{\| x-y \|}{\sqrt{\rho(x,y)}} \mbox{ and } \rho(x,a)= \| x-a \|^2 + (1- \| a \|^2)(1- \| x \|^2).
 \end{equation}
Consider transformations of the unit ball of the following form
\begin{equation} 
\label{Mobius_ball}
h_a(x)= A \frac{a \| x-a \|^2+(1- \| a \|^2)(a-x)}{\rho(x,a)},
\end{equation}
where $a \in \mathbb{B}^d$ and $A$ is an orthogonal transformation of the Euclidean space $\mathbb{R}^d$.

Transformations \eqref{Mobius_ball} map the unit ball in $\mathbb{R}^d$ onto itself. They are named M\" obius (or conformal) transformations. 

It is easy to verify that $h_c^{-1}(x)=h_c(x)$ for each $c\in \mathbb{B}^d$. These transformations constitute the isometry group for the manifold ${\mathbb B}^d$. We denote it by $Isom(\mathbb{B}^d)$. This group is isomorphic to the Lorentz group $SO^+(d,1)$.




Jacobian of the mapping $y=h_a(x)$ is given by  
\begin{equation}
\label{jaco}
{\cal J}(y,x)=\frac{1- \| a \|^2}{\rho(a,x)^n}=\frac{(1- \| y \|^2)^n}{(1- \| x \|^2)^n}.
\end{equation}

\subsection{Cauchy distributions on spheres in real vector spaces}

Our first model is the following family of probability distributions on spheres introduced in \cite{KMcC}:
\begin{equation}
\label{conf_nat_dens}
p_{sC}(x;a) = \frac{\Gamma(d/2)}{2 \pi^{d/2}} \left( \frac{1 - \| a \|^2}{\| x-a \|^2} \right)^{d-1}, \quad \| x \| = 1, \, \| a \|<1.
\end{equation}
This family of densities is parametrized by a point $a$ in the unit ball. Following \cite{KMcC} we will refer to the distributions defined by \eqref{conf_nat_dens} as {\it spherical Cauchy} distributions. \footnote{Notice that densities of the form \eqref{conf_nat_dens} arise in the geometric theory of PDE's as hyperbolic Poisson kernels. They are eigenfunctions of the hyperbolic Laplace-Beltrami operator for the minimal eigenvalue.} Accordingly, we will denote this family by $sC(a)$.

The family $sC(a)$ includes the uniform on ${\mathbb S}^{d-1}$ for $a=0$. On the other extreme, delta distributions appear in the limit $\|a\| \to 1$.

\subsubsection{Properties of the Cauchy family on spheres}

We briefly outline some properties of the family $sC(a)$ that are relevant for our analysis. Most of these properties have been proven in \cite{JK}. 

\begin{enumerate}

\item [i)] {\bf Conformal invariance.} \,
 The family $sC(a)$ is invariant for actions of the group of conformal transformations \eqref{Mobius_ball}. Moreover, this group acts transitively on the family $sC(a)$. Namely


\begin{prop}
If $X \sim sC(y)$ then $g_a(X) \sim sC(g_a(y))$.
\end{prop}

\item [ii)] {\bf Centroid.} \, Let $X \sim sC(a)$. Then the first moment (centroid) of $X$ is a point in the unit ball given by
$$
\langle X \rangle = \int_{{\mathbb S}^d} x \, p_{sC}(x;C) \, d \sigma(x) = \mu_{d-1}(\| a \|) a, \mbox{  where}
$$
$$\mu_{d-1}(y) = \frac{1+y^2}{2 y^2} \left[ 1 - \frac{1-y^2}{1+y^2} F \left\{\frac{1}{2};\frac{d-1}{2};\frac{d+1}{2};-\frac{4 y^2}{(1-y^2)^2} \right\} \right]
$$
and $F\left\{ \cdot;\cdot;\cdot;\cdot\right\}$ denotes the hypergeometric series.

We refer to \cite{KMcC} for details and formulae for higher order moments.

\item [iii)] {\bf Sampling.} \, In order to generate random variate from $sC(a)$ it suffices to generate a uniformly distributed point $y$ on the sphere ${\mathbb S}^d$ and map it using \eqref{Mobius_ball}. Indeed, if $y \sim sC(0)$ then $g_a(y) \sim sC(a)$.

\item[iv)] {\bf Fisher information.} 

\begin{prop}
\label{Fisher_info_Cauchy}
Fisher information matrix for the family $sC(a)$ is given by
\begin{equation}
\label{Fisher_conf_nat}
F_{sC}(a) = - {\mathbb E} \left \{ \frac{\partial \log p_{sC}(x;a)}{\partial a} \frac{\partial \log p_{sC}(x;a)}{\partial a^T} \right \} = \frac{4}{(1 - \| a \|^2)^2} \frac{(d-1)^2}{d} I.
\end{equation}
\end{prop}
The formula \eqref{Fisher_conf_nat} has been reported in \cite{KMcC}. We provide the proof in Appendix.

The expression \eqref{Fisher_conf_nat} unveils remarkable fact that the Fisher information metric turns the family $sC(a)$ into the Poincar\' e ball (up to the dimension-adjustment multiplier).

\item[v)] {\bf Circular model in the plane.} \, 
For $d=2$ the densities \eqref{conf_nat_dens} reduce to circular (or wrapped) Cauchy distributions, introduced in \cite{McCullagh}. In this case $Isom(\mathbb{B}^2)$ is the group of disc-preserving M\" obius transformations in the complex plane.

\end{enumerate}

\section{Spherical Cauchy natural search gradients}

Suppose that we have $M$ evaluation points $y_i \in {\mathbb S}^{d-1}$ and evaluations of the fitness function $f(y_i)$ at these points. Introduce a suitable fitness shaping rule $F[f(\cdot)]$ using common methods in NES \cite{Wierstra,HA}.
 
Denote the shaped evaluations $\gamma_j = F[f(y_j)]$. Without loss of generality, assume that all weights $\gamma_j$ are non-negative.

Then the weighted Euclidean ("vanilla") average score for the function $J(a) = {\mathbb E}_a [f(y)]$ reads:
\begin{equation}
\label{sC_policy_gradients}
\begin{split}
& \nabla_{Euc} J(a) \approx \frac{1}{M} \frac{\partial}{\partial a} \sum_{j=1}^M \gamma_j \log \left( \frac{1- \| a \|^2}{\| y_j-a\|^2} \right)^{d-1} = \frac{d-1}{M} \sum_{j=1}^M \gamma_j \left( \nabla_a \log (1- \| a \|^2) - \nabla_a \log \| y_j - a\|^2 \right)  \nonumber \\
& = \frac{d-1}{M} \sum_{j=1}^M \gamma_j \left( \frac{-2a}{1-\| a \|^2} + \frac{2a - 2y_j}{\| y_j-a \|^2} \right) = - \frac{2d-2}{M} \sum_{j=1}^M \gamma_j \frac{a \| y_j - a\|^2 + (y_j-a) (1 - \| a \|^2)}{(1 - \| a\|^2)\|y_j - a\|^2} \nonumber \\
& = - 2 \frac{d-1}{M} (1-\| a \|^2)^{-1} \sum_{j=1}^M \gamma_j \frac{a \| y_j - a\|^2 + (y_j-a) (1 - \| a \|^2)}{\|y_j - a\|^2} = - 2 \frac{d-1}{M} (1 - \| a\|^2)^{-1} \sum_{j=1}^M \gamma_j \, g_a(y_j).
\end{split}
\end{equation}
where $g_a$ is given by \eqref{Mobius_ball}.

Multiplying by the inverse of the Fisher information matrix \eqref{Fisher_conf_nat} we obtain natural gradients:
\begin{equation}
\label{sC_nat_policy_gradients}
\begin{split}
& \tilde \nabla J(a) = F_{sC}^{-1}(a) \nabla_{Euc} J(a) \approx - 2 \frac{d-1}{M} \frac{d}{4 (d-1)^2} (1-\| a \|^2) \sum_{j=1}^M \gamma_j \, g_a(y_j) \\
& = - \frac{d}{2(d-1)M} (1-\| a \|^2) \sum_{j=1}^M \gamma_j \, g_a(y_j). 
\end{split}
\end{equation}

Therefore, the natural gradient update follows the ODE
\begin{equation}
\label{nat_grad_ODE}
\frac{da}{dt} = - (1 - \| a \|^2) \frac{\alpha}{M} \sum_{j=1}^M \gamma_j \, g_a(y_j).
\end{equation}
where $\alpha$ is the optimization step.

ODE \eqref{nat_grad_ODE} is the information-geometric flow (IGO flow) for the family $sC(a)$. Notice that the stationary point $a$ of the IGO flow \eqref{nat_grad_ODE} satisfies:
\begin{equation}
\label{conf_barycenter}
\sum_{j=1}^M \gamma_j \, g_a(y_j) = 0.
\end{equation}
Notice that \eqref{conf_barycenter} is the equation w. r. to $a$, as $\gamma_j$ and $y_j$ are fixed. This equation unveils interpretation of this stationary point as conformal barycenter of a probability measure on the sphere as introduced in \cite{DE}. In order to explain this, notice that transformation $g_a$ maps the configuration of points $x_1,\dots,x_M$ with weights $w_1,\dots,w_M$ into a balanced configuration. \footnote{Configuration of points on the sphere is said to be {\it balanced}, if their Euclidean mean (centroid) is zero (center of the ball).} As shown in \cite{DE}, under the assumption that the probability measure does not contain atoms with weights higher than $1/2$, there is a unique (up to a rotation) transformation $g_a$ mapping the initial configuration into a balanced one. 

Now, consider the probability measure concentrated at $M$ atoms $y_1,\dots,y_M$ weighted by $\gamma_1,\dots,\gamma_M$. As explained above, if $\gamma_i<1/2$ for $i=1,\dots,M$ the equation \eqref{conf_barycenter} has a unique solution $a = a^*$. This solution is conformal barycenter of the probability measure.

\subsection{Kuramoto oscillators on spheres compute spherical Cauchy natural search gradients}

The next question is regarding the algorithm for (distributed) computation of natural search gradients \eqref{nat_grad_ODE}. To address this question we introduce the generalized Kuramoto model describing an ensemble of identical globally coupled generalized oscillators on the sphere \cite{LMS}:
\begin{equation}
\label{Kuramoto_real}
\dot x_j = G - (G^T x_j) x_j, \quad j=1,\dots,M.
\end{equation} 
Here, $x_j(t)$ are unit vectors in ${\mathbb R}^{d+1}$ representing positions of oscillators on the sphere at the moment $t>0$ and $G \equiv G(x_1,\dots,x_N)$ is a coupling vector-valued function. We will consider functions of the form 
\begin{equation}
\label{coupling}
G(x_1,\dots,x_M) = \frac{K}{M} \sum_{i=1}^M w_i x_i.
\end{equation}
Coefficients $w_i > 0$ can be interpreted as weights of oscillators, while $K>0$ is the global coupling strength.

Differentiating the expression $\| x \|^2 = x^T x$ one can verify that dynamics \eqref{Kuramoto_real} preserve the unit sphere. In other words, if initial positions $x_1(0),\dots,x_M(0)$ are unit vectors, then the solutions $x_1(t),\dots,x_M(t)$ of \eqref{Kuramoto_real} will be unit vectors at any $t>0$.

Moreover, it is shown in \cite{LMS} that oscillators $x_i(t)$ evolve by the actions of the group of transformations \eqref{Mobius_ball}. More precisely, the following assertion holds.

\begin{prop}
Consider the system \eqref{Kuramoto_real} evolving from initial positions $x_1(0),\dots,x_M(0)$ on the sphere ${\mathbb S}^{d-1}$. Then there exists a one-parametric family of conformal mappings $g_{a(t)}$ such that
$$
x_j(t) = g_{a(t)}(x_j(0)), \quad j=1,\dots,M, \; t>0.
$$
Moreover, the parameter $a$ of the transformation \eqref{Mobius_ball} evolves by the following ODE:
\begin{equation}
\label{low_dim_Kuramoto_real}
\dot a = - \frac{K}{2M} (1 - \| a \|^2) \sum_{i=1}^M w_i g_a(x_j(0)). 
\end{equation}
\end{prop}

We refer to \cite{LMS} for derivation of the above Proposition.

Comparing \eqref{low_dim_Kuramoto_real} with \eqref{nat_grad_ODE} we find that one needs to set initial positions $x_i(0) = y_i$ and weights $w_i = \gamma_i$. Then dynamics \eqref{Kuramoto_real}-\eqref{coupling} approximate the natural search gradients for the family $sC(a)$.

\subsection{Algorithms}

As shown above, the Kuramoto oscillators \eqref{Kuramoto_real} generate the natural gradient flow in hyperbolic balls and thus provide the computational tool for calculating natural search gradients for the family $sC(a)$. As emphasized in \cite{OAAH} an IGO flow gives the rise to different IGO optimization algorithms (we also refer to \cite{HA} for a detailed overview of challenges and receipts that arise in practice). Here, we present two IGO algorithms which are based on the framework explained above.    

\begin{algorithm}[H]
\caption{Small updates along the natural gradient}
\begin{algorithmic}[1]

\State Sample $M$ points $y_1,\dots,y_M$ from the uniform distribution on ${\mathbb S}^{d-1}$.

\State Evaluate $f(y_1),\dots,f(y_M)$. Assign the weights $\gamma_1,\dots,\gamma_M$ to oscillators based on evaluations (using a suitable fitness shaping scheme). 

\State Set $K$ and $w_i = \gamma_i$ for $i=1,\dots,M$.

\State Run Kuramoto dynamics \eqref{Kuramoto_real}-\eqref{coupling} from initial positions $x_1(0) = y_1,\dots,x_M(0)=y_M$ on a small time interval $t \in [0,\delta t]$, where $\delta t$ is a hyperparameter of the algorithm. Denote $y_1=x_1(\delta t),\dots,y_M = x_M(\delta t).$ 

\State Denote $y_1 = x_1(\delta t),\dots,y_M = x_M(\delta t)$. Iterate steps 2)-4).

\end{algorithmic}
\end{algorithm}

Algorithm 1 performs small updates in the space $sC(a)$ in directions opposite to the natural gradient. Another option is to perform the maximum likelihood estimation based on evaluations under the assumption of the Cauchy family on spheres. This approach explicitly leverages conformal barycenters of probability measures on spheres.

\begin{algorithm}[H]
\caption{Maximum likelihood updates under the model $sC(a)$ by computing stationary points of the IGO flow \eqref{nat_grad_ODE}}
\begin{algorithmic}[1]

\State Sample $M$ points $y_1,\dots,y_M$ from the uniform distribution on ${\mathbb S}^{d-1}$.

\State Evaluate $f(y_1),\dots,f(y_M)$. Assign the weights $\gamma_1,\dots,\gamma_M$ to oscillators based on evaluations. To this end, use a certain suitable shaping scheme to ensure that $1/2>\gamma_i>0$ for all $i=1,\dots,M$. 

\State Denote by $\mu$ the probability measure concentrated at $M$ points $y_1,\dots,y_M$ with the corresponding weights $\gamma_1,\dots,\gamma_M$. Find the conformal barycenter of the measure $\mu$. Denote this conformal barycenter by $a_\mu$.

\State Sample $M$ points $y_1,\dots,y_M$ from the probability distribution $sC(a_\mu)$. Iterate steps 2) and 3).

\end{algorithmic}
\end{algorithm}
 
Algorithm 2 requires computation of conformal baryecenters of probability measures. Hence, this algorithm is the maximum likelihood update under the model assumption $sC(a)$. Indeed, given a probability measure on $\mathbb S^{d-1}$, its conformal barycenter is the maximum likelihood estimation for $a$ under the model $sC(a)$, see \cite{CHMR}.
 
Conformal barycenter is the (unique) solution of the equation \eqref{conf_barycenter}. Hence, it is a stationary point of the gradient flow \eqref{nat_grad_ODE}. The discussion in the previous subsection shows that this point can be computed by running the Kuramoto model with repulsive coupling $K<0$ (see \cite{CEM,Jacimovic} for the discussion in the disc.) Another option is the Newton's numerical scheme proposed in \cite{CS}. Both methods can be implemented in a linear time w. r. to the number of atoms $M$. 
 
\section{Holomorphically natural family of probability distributions on spheres in complex vector spaces}

In this section we examine another search policy where spheres are considered as manifolds in complex vector spaces. We will denote the complex vector space by $\mathbb C^m$ and the unit sphere 
$$
\mathbb S^{2m-1} = \{ (z_1,\dots,z_m) , |z_1|^2 + \cdots + |z_m|^2 = 1, \; z_i \in \mathbb{C} \}.
$$
In our notation we accentuate the real dimension $2m-1$ of the sphere.

We proceed with some preliminary concepts on hyperbolic balls in complex vector spaces.

For two vectors $\xi=(\xi_1 \dots \xi_m)$ and $\eta = (\eta_1 \dots \eta_m)$ the notation $\xi^\dagger \eta$ stands for the Hermitian scalar product $\xi^\dagger \eta = \sum \xi_i \bar \eta_i$. Norm of the vector $\xi$ is denoted by $\| \xi \| = \sqrt{\xi^\dagger \xi}$.

\subsection{Bergman balls and their isometry groups}


\begin{definition}
Let $\mathbb{B}_m$ be the unit ball in $\mathbb{C}^m$ equipped with the metric tensor defined by 
$$
g_z(u,v) = (B(\zeta)u)^\dagger v, \; u, v\in \mathbb{C}^m, \; \zeta \in \mathbb{B}_m.
$$
Here $
B(\zeta)=\{ b_{ij}(\zeta) \}_{i,j=1}^m \quad \mbox{ and } \quad b_{ij}(\zeta)= \frac{1}{m+1}\frac{\partial^2}{\partial \overline{{\zeta_i}}\partial \zeta_j}K(\zeta,\zeta),
$ 
where 
$$
K(\zeta,w)=\frac{1}{m+1}\frac{1}{(1- \zeta^\dagger w)^{m+1}}
$$ 
is the Bergman kernel \cite{kobayashi,kezu}.

The Riemannian manifold $(\mathbb{B}_m, g)$ is named the Bergman  ball.
\end{definition}

\begin{rem}
In the above definition we introduced the Bergman ball metric from Bergman kernels. This metric tensor also allows for a relatively simple explicit expression \cite{Rudin}
\begin{equation}
\label{Bergman_matric}
B(\zeta) = \frac{m}{(1 -\| \zeta \|^2)^2} \left[ (1 - \| \zeta \|^2) I + \zeta \zeta^\dagger \right]
\end{equation}
\end{rem}
Notice that $\zeta \zeta^\dagger$ is a rank-one matrix.

Bergman balls have constant negative sectional curvature, see \cite{HL}.

Let $P_a$ be the
orthogonal projection of $\mathbb{C}^m$ onto one-dimensional (complex) subspace $Lin \{a \}$ generated by $a$, and let $Q=Q_a =
I - P_a$ be the projection onto the orthogonal complement of $Lin\{a\}$. Explicitly, 
$$
P_0 = 0 \mbox{   and  }  P=P_a(\zeta) =\frac{(\zeta^\dagger a) a}{ \| a\|}.
$$
 Set $s_a = (1 - \|a\|^2)^{1/2}$ and consider the map
\begin{equation}
\label{Bergman_transf}
m_a(\zeta) = R \frac{a-P_a \zeta-s_a Q_a \zeta}{1- \zeta^\dagger a}, \qquad R \in U(m).
\end{equation}
Mappings of the form \eqref{Bergman_transf} consitute the group of holomorphic automorphisms of the unit ball $\mathbb{B}_m \subset \mathbb{C}^m$. We will denote this group by $Aut(\mathbb{B}_m)$. This group is isomorphic to the Lie group $SU(m,1)$. 

We have chosen parametrization \eqref{Bergman_transf} to have $m_a^{-1}=m_a$. 


The Bergman metric on $\mathbb B_m$ can be introduced by the following formula (see \cite[Proposition~1.21]{kezu}): 
\begin{equation}
\label{bergmet}
d_B(\zeta,w)=\frac{1}{2}\log\frac{1+ \| m_w(\zeta) \|}{1- \| m_w(\zeta) \|}.
\end{equation}
If $\Omega= \{\zeta \in \mathbb{C}^n: \zeta^\dagger a \neq 1\}$,  then the map $p_a$  is holomorphic in $\Omega$.

It is well-known that every automorphism $q$ of the unit ball is an isometry w.r. to the Bergman metric, that is: $d_B(\zeta,w)=d_B(q(\zeta),q(w))$.

We also point out the formula 
\begin{equation}
\label{phia}
(1-\| m_a(z)\|^2)=\frac{(1-\|\zeta \|^2)(1-\|a\|^2)}{\|1- a^\dagger \zeta \|^2}\end{equation} and the expression for the Jacobian
$$
{\cal J}(\zeta,m_a)=\left(\frac{1-\|m_a(\zeta)\|^2}{1-\|\zeta\|^2}\right)^{m+1}=\left(\frac{1-\|a\|^2}{\|1-\zeta^\dagger a\|^2}\right)^{m+1}.
$$

Along with general automorphisms of $\mathbb{B}_m$ given by \eqref{Bergman_transf}, we will also work with the subset of holomorphic automorphisms of $\mathbb B_m$ that have a simpler form. Namely, consider maps of the form
\begin{equation}
\label{holomorphic_reduced}
\phi_{- R \alpha}(\zeta) = R \frac{\zeta - \alpha}{1 - \alpha^\dagger \zeta}, \mbox{  where } \alpha \in \mathbb B_m, \quad R \in U(m).
\end{equation}
We use this notation to emphasize that $\phi_{-R \alpha}(0) = -R \alpha$.

\subsection{Bergman probability distributions on spheres in complex vector spaces}

We will consider the family of probability distributions on ${\mathbb S}^{2m-1}$ defined by densities 
\begin{equation}
\label{Bergman_dens}
p_{sB}(z;\zeta) = C_m \frac{(1-\| \zeta \|^2)^m}{\|1 - \zeta^\dagger z \|^{2m}}, \quad \| z \| = 1, \; \| \zeta \| <1.
\end{equation}

Here, $C_m$ is normalizing constant depending on the choice of the metric element $\sigma$ on the sphere. One may set $\mathbb C_m = \frac{(m-1)!}{2 \pi^m}$ corresponding to the standard choice. However, the constant $C_m$ is not relevant for further considerations.

We will refer to the family defined by densities \eqref{Bergman_dens} as {\it Bergman distributions on spheres in complex vector spaces} and denote it by $sB(\zeta)$. This term emphasizes the close relation with the Bergman metric in the unit ball and Bergman kernels. 

The family $sB(\zeta)$ includes the uniform distribution on $\mathbb S^{2m-1}$ for $\zeta = 0$, while delta distributions appear in the limit $\| \zeta \| \to 1$.

We also point out that the expression \eqref{Bergman_dens} are known as Poisson-Szeg\" o kernels in complex analysis \cite{Krantz}. It has been recently proposed by the author as a statistical model for ML using Kuramoto ensembles, see \cite{JH}.

\subsubsection{Properties of the Bergman family on spheres}

We further list some of the essential properties of the Bergman family of probability distributions on complex spheres. Most of the statements below are proven in \cite{JK2}.

\begin{enumerate}

\item[i)] {\bf Holomorphic invariance.}

\begin{prop}
The family $sB(\zeta)$ is invariant w. r. to actions of the group $Aut(\mathbb{B}_m)$. Moreover, this group acts transtivelly on $sB(\zeta)$.
\end{prop}

Densities $sB(\zeta)$ are orbits of the group $Aut(\mathbb{B}_m)$ for the ground state described by the density of the uniform distribution $sB(0)$ on $\mathbb{B}_m$.

\item[ii)] {\bf The first moment.}
Let $Z \sim sB(\zeta)$. Then the first moment (centroid) of the random variable $Z$ is $\langle Z \rangle = \zeta$. This can be easily verified by evaluating the corresponding integral over the sphere ${\mathbb S}^{2m-1}$.

\item[iii)] {\bf Sampling.}
Holomorphic invariance enables efficient random variate generation from the Bergman family. In order to generate a random variate from $sB(a)$, we sample a random unit vector $v$ and apply the transformation \eqref{Bergman_transf}. Then $m_a(v) \sim sB(a)$.

\item [iv)] {\bf Fisher information.}

\begin{prop}
\label{Fisher_info_Bergman_prop}
The Fisher information matrix for the family $sB(\zeta)$ is given by
\begin{equation}
\label{Fisher_Bergman}
F_{sB}(\zeta) = \frac{m}{(1 - \| \zeta \|^2)^2}\left[ (1 - \| \zeta \|^2) I + \zeta \zeta^\dagger \right].
\end{equation}
\end{prop}

The above Proposition demonstrates that the Fisher information metric for the Bergman family is precisely Bergman metric in the unit ball. The proof is given in Appendix A.

\begin{prop}
\label{inverse_Fisher_Bergman_prop}
Inverse of the Fisher information matrix \eqref{Fisher_Bergman} is given by
\begin{equation}
\label{inverse_Fisher_Bergman}
F_{sB}(\zeta)^{-1} = \frac{1 - \| \zeta \|^2}{m} (I - \zeta \zeta^\dagger).
\end{equation}
\end{prop}

The proof relies on the formula for the inverse of a rank-one update of the identity matrix proven by Sherman and Morisson in \cite{SM}. Rigorous derivations are written out in Appendix A.

\item[v)] {\bf Circular model.}
For $m=1$ \eqref{Bergman_dens} reduce to densities of wrapped Cauchy distributions on the circle in the complex plane. We can see that for the circle models \eqref{conf_nat_dens} and \eqref{Bergman_dens} are equivalent. Accordingly, both isometry transformations \eqref{Mobius_ball} and \eqref{Bergman_transf} reduce to standard disc preserving M\" obius transformations $z \to e^{i \varphi} \frac{\beta - z}{1 - \bar \beta z}$ in the complex plane in dimension $m = 1$. Emphasize that in complex dimensions higher then $m=1$ \eqref{Mobius_ball} and \eqref{Bergman_transf} are not equivalent. 

\end{enumerate}

\section{Bergman natural search gradients}

In this Section we will compute natural search gradients of a function defined on the complex space. To that aim, formalism of the Wirtinger derivatives provides a convenient and elegant notational system\cite{Hormander}. For $F(z) = F(z_1,\dots,z_m)$, where $z_k = x_k + i y_k$ define partial derivatives
$$
\frac{\partial F}{\partial z_k} = \frac{1}{2} \left( \frac{\partial F}{\partial x_k} - i \frac{\partial F}{\partial y_k} \right) \mbox{    and    } \frac{\partial F}{\partial \bar z_k} = \frac{1}{2} \left( \frac{\partial F}{\partial x_k} + i \frac{\partial F}{\partial y_k} \right).
$$
The Euclidean ("vanilla") gradient of $F$ reads
$$
\nabla_{Euc} F(z) = \left( \frac{\partial F}{\partial \bar z_1} \, \frac{\partial F}{\partial \bar z_2} \cdots \frac{\partial F}{\partial \bar z_m} \right)^T \equiv \frac{\partial F}{\partial z^\dagger}.
$$  

Now we can compute Euclidean gradient of the function $J(\zeta)$:
\begin{equation}
\label{Bergman_scores}
\begin{split}
& \nabla_{Euc} J(\zeta) \approx \frac{\partial}{\partial \zeta^\dagger} \frac{1}{M} \sum_{j=1}^M \gamma_j \log \frac{(1-\| \zeta \|^2)^m}{\|1 - \zeta^\dagger z_j \|^{2m}} = \frac{m}{M} \frac{\partial}{\partial \zeta^\dagger} \sum_{j=1}^M \gamma_j \log \frac{1 - \| \zeta \|^2}{(1 - z_j^\dagger \zeta)(1 - \zeta^\dagger z_j)} \\ 
& =  \frac{m}{M} \sum_{j=1}^M \gamma_j (\nabla_{\zeta^\dagger} \log(1 - \zeta^\dagger \zeta) -  \nabla_{\zeta^\dagger} \log (1 - z_j^\dagger \zeta) - \nabla_{\zeta^\dagger} \log (1 - \zeta^\dagger z_j)) \\
& = \frac{m}{M} \sum_{j=1}^M \gamma_j \left( \frac{z_j}{1 - \zeta^\dagger z_j} - \frac{\zeta}{1 - \| \zeta\|^2} \right).
\end{split}
\end{equation}

Multiply by inverse of the Fisher matrix \eqref{inverse_Fisher_Bergman} to obtain the natural gradient:
\begin{equation}
\label{Bergman_scores_natural}
\begin{split}
& \tilde \nabla J(\zeta) = F_{sB}(\zeta)^{-1} \nabla_{Euc} J(\zeta) \approx \frac{m}{M} F_{sB}(\zeta)^{-1} \sum_{j=1}^M \gamma_j \left( \frac{z_j}{1 - \zeta^\dagger z_j} - \frac{\zeta}{1 - \| \zeta\|^2} \right) \\
& = \frac{m}{M} \frac{1 - \| \zeta \|^2}{m} (I - \zeta \zeta^\dagger) \sum_{j=1}^M \gamma_j \left( \frac{z_j}{1 - \zeta^\dagger z_j}  - \frac{\zeta}{1 - \| \zeta \|^2}\right) \\
& = \frac{1}{M} (1 -\| \zeta \|^2) \sum_{j=1}^M \gamma_j \left( \frac{z_j}{1 - \zeta^\dagger z_j} - \frac{\zeta}{1 - \| \zeta \|^2} - \frac{\zeta (\zeta^\dagger z_j)}{1 - \zeta^\dagger z_j} + \frac{(\zeta \zeta^\dagger) \zeta}{1 - \| \zeta \|^2} \right) \\
& = \frac{1}{M} (1 -\| \zeta \|^2) \sum_{j=1}^M \gamma_j \left( \frac{z_j}{1 - \zeta^\dagger z_j} - \frac{ (\zeta^\dagger z_j) \zeta}{1 - \zeta^\dagger z_j} - \zeta \right) \\
& = \frac{1}{M} (1 -\| \zeta \|^2) \sum_{j=1}^M \gamma_j \frac{z_j - (\zeta^\dagger z_j) \zeta - \zeta + \zeta (\zeta^\dagger z_j)}{1 - \zeta^\dagger z_j} \\
& = \frac{1}{M} (1 - \| \zeta \|^2) \sum_{j=1}^M \gamma_j \frac{z_j - \zeta}{1 - \zeta^\dagger z_j} = \frac{1}{M} (1 - \| \zeta \|^2) \sum_{j=1}^M \gamma_j \phi_{-I \zeta}(z_j),
\end{split}
\end{equation}
where $\phi_{-I \zeta}(\cdot)$ is holomorphic mapping of the form \eqref{holomorphic_reduced} with $R = I$ ($I$ denotes the identity transformation).

Therefore the natural gradient flow for update of the Bergman distribution \eqref{Bergman_dens} reads
\begin{equation}
\label{nat_grad_flow_complex}
\dot \zeta = \frac{\kappa}{M} (1 - \| \zeta \|^2) \sum_{j=1}^M \gamma_j \phi_{-I \zeta(t)}(z_j)
\end{equation}
where $\kappa$ is an optimization step.

In terminology of \cite{OAAH}, ODE \eqref{nat_grad_flow_complex} is the IGO flow for the family $sB(\zeta)$.


\subsection{Projective Kuramoto model on spheres in complex vector spaces computes Bergman search gradients}

The standard Kuramoto model describing collective motions of $M$ generalized oscillators on the complex sphere reads $\dot z_j = H z_j + Z - (z_j^\dagger Z) z_j, \; j=1,\dots,M$, see \cite{LMS}. We assume a common intrinsic frequency given by an anti-Hermitian matrix $H$ and a global coupling $Z$. One might expect that, in analogy with Section 4, this model provides a convenient tool for computation of policy gradients \eqref{Bergman_scores_natural}. However, more careful considerations demonstrate that a slightly different Kuramoto model is more suitable for our computations.

We consider a projective Kuramoto model on ${\mathbb S}^{2m-1}$ given by 
\begin{equation}
\label{Kuramoto_complex}
\dot z_j = Z - (z_j Z^\dagger) z_j, \quad j=1,\dots,M
\end{equation}
Here, $K \in {\mathbb R}$ is the coupling strength and $Z(z_1,\dots,z_M)$ is weighted centroid of oscillators $z_1,\dots,z_M$:
\begin{equation}
\label{coupling_complex}
Z(z_1,\dots,z_M) = \frac{K}{M} \sum_{j=1}^M w_j z_j, \; w_j \in {\mathbb R}
\end{equation}
It is easy to verify that dynamics \eqref{Kuramoto_complex}-\eqref{coupling_complex} preserve the unit sphere, that is $z_1(0),\dots,z_M(0) \in \mathbb S^{2m-1}$ implies that $z_1(t),\dots,z_M(t) \in \mathbb S^{2m-1}$ for all $t>0$.

We now assert the Proposition describing low-dimensional dynamics generated by the collective evolution \eqref{Kuramoto_complex}-\eqref{coupling_complex}.

\begin{prop}
\label{low_dim_Bergman}

Consider the system \eqref{Kuramoto_complex}-\eqref{coupling_complex} evolving from initial positions $z_1(0),\dots,z_M(0)$ on ${\mathbb S}^{2m-1}$. Then there exists a one-parametric family of holomorphic mappings $\phi_{-R \alpha}(\cdot)$ of the form \eqref{holomorphic_reduced} such that
$z_j(t) = \phi_{-R \zeta(t)}(z_j(0)), \quad j=1,\dots,M, \; t>0$.

Moreover, parameters $\zeta(t)$ and $R(t)$ of the transformation \eqref{holomorphic_reduced} evolve by the following ODE's
\begin{equation}
\label{low_dim_Kuramoto_complex}
\dot \zeta = - \frac{K}{M} (1 - \| \zeta(t) \|^2) \sum_{j=1}^M w_j \phi_{-I \zeta(t)}(z_j(0)).
\end{equation}
\begin{equation}
\label{geom_phase}
\dot R = (Z^\dagger R \zeta) R - Z \zeta^\dagger
\end{equation}
\end{prop}

\begin{proof}

We will follow the derivation from \cite{CEM} for the dimension $m=1$ (complex disc). Assume that $z_j(t)$ evolve by the action of transformations of the form \eqref{holomorphic_reduced}
\begin{equation}
\label{z_j(t)}
z_j(t) = R(t) \frac{z_j(0) - \zeta(t)}{1 - \zeta(t)^\dagger z_j(0)}.
\end{equation}
Then
\begin{equation}
\label{z_j(t)_diff}
\begin{split}
& \dot z_j(t) = \dot R(t) \frac{z_j(0)-\zeta(t)}{1 - \zeta(t)^\dagger z_j(0)} + R(t) \frac{- \dot \zeta(t)}{1 - \zeta(t)^\dagger z_j(0)} + R(t) \frac{(z_j(0)-\zeta(t))(\dot \zeta(t)^\dagger z_j(0))}{(1 - \zeta(t)^\dagger z_j(0))^2} \\
& = (\dot R(t) R^\dagger(t)) z_j(t) + \frac{-R(t) \dot \zeta(t)}{1 - \zeta(t)^\dagger z_j(0)} + z_j(t) \frac{\dot \zeta(t)^\dagger z_j(0)}{1 - \zeta(t)^\dagger z_j(0)}
\end{split}
\end{equation}
Now, refer to Lemma in Appendix B and set $v \equiv z_j(t), \beta \equiv z_j(0), w \equiv \zeta(t)$ in \eqref{identity1}. In these notations the two identities read
$$
\frac{1}{1 - \zeta(t)^\dagger z_j(0)} = \frac{1 + z_j(t)^\dagger R(t)^\dagger \zeta(t)}{1 - \| \zeta(t) \|^2} \; \mbox{   and  } \;
z_j(0) = \frac{R^\dagger z_j(t) + \zeta(t)}{1 + \zeta(t)^\dagger R(t)^\dagger z_j(t)}.
$$
Substituting in \eqref{z_j(t)_diff} we find that (as now we excluded $z_j(0)$, all variables $R, \zeta$ and $z_j$ are functions of time, so we will omit the dependence on time in notations):
\begin{equation}
\label{z_j_ODE_1}
\dot z_j = \dot R R^\dagger z_j + \frac{1 + \zeta^\dagger R^\dagger z_j}{1 - \| \zeta \|^2} \left[ - R \dot \zeta + z_j \dot \zeta^\dagger \frac{R^\dagger z_j + \zeta}{1 + \zeta^\dagger R^\dagger z_j} \right].
\end{equation}

Further, denote $\Omega(t) = \dot R(t) R(t)^\dagger$ and rearrange the last ODE to get
\begin{equation}
\label{z_j_ODE_2}
\begin{split}
& \dot z_j = \Omega z_j + \frac{1}{1 - \| \zeta \|^2} [-(1 + \zeta^\dagger R^\dagger z_j)R \dot \zeta + z_j (\dot \zeta^\dagger R^\dagger z_j + \dot \zeta^\dagger \zeta) ] \\
& = \Omega z_j + \frac{1}{1 - \| \zeta \|^2} [-R \dot \zeta - (\zeta^\dagger R^\dagger z_j)R \dot \zeta + z_j(\dot \zeta^\dagger R^\dagger z_j) + z_j (\dot \zeta^\dagger \zeta)] \\
& = \frac{-R \dot \zeta}{1 - \| \zeta \|^2} + \Omega z_j + \frac{\dot \zeta^\dagger \zeta}{1 - \| \zeta \|^2} z_j - \frac{(\zeta^\dagger R^\dagger z_j) R \dot \zeta}{1 - \| \zeta \|^2} + \frac{(z_j \dot \zeta^\dagger R^\dagger) z_j}{1-\| \zeta \|^2}.
\end{split}
\end{equation}

Comparing \eqref{z_j_ODE_2} with \eqref{Kuramoto_complex} we immediately see that 
\begin{equation}
\label{Z}
Z = - \frac{R \dot \zeta}{1 - \| \zeta \|^2}.
\end{equation}
Solving the last expression w. r. to $\dot \zeta$ we obtain the ODE  for $\zeta$:
$$
\dot \zeta = - (1 - \| \zeta \|^2) R^\dagger Z = - \frac{K}{M} (1 - \| \zeta \|^2) R^\dagger \sum_{j=1}^M w_j z_j(t).
$$
Substitute \eqref{z_j(t)} into the last equation to get:
\begin{equation}
\label{zeta}
\begin{split}
 \dot \zeta = - \frac{K}{M} (1 - \| \zeta \|^2) R^\dagger \sum_{j=1}^M w_j R \frac{z_j(0) - \zeta}{1 - \zeta^\dagger z_j(0)}
 = - \frac{K}{M} (1- \| \zeta \|^2) \sum_{j=1}^M w_j \phi_{-I \zeta}(z_j(0)).
\end{split}
\end{equation}

In order to derive ODE for $R$ we need set the linear term in \eqref{z_j_ODE_2} to zero. But, before that we substitute the equality $\dot \zeta = - (1 - \| \zeta \|^2)R^\dagger Z$ in \eqref{z_j_ODE_2}. This yields:
\begin{equation}
\label{z_j_ODE3}
\dot z_j = Z + \Omega z_j - (Z^\dagger R \zeta) z_j + (\zeta^\dagger R^\dagger z_j) Z - (z_j Z^\dagger) z_j.
\end{equation}
Using that $\zeta^\dagger R^\dagger z_j$ is a scalar, we extract the matrix terms multiplying $z_j$ and set them to zero. This yields:
$$
\Omega - Z^\dagger R \zeta I + Z \zeta^\dagger R^\dagger = 0.
$$
or, recalling that $\Omega = \dot R R^\dagger$: 
$$
\dot R = (Z^\dagger R \zeta) R - Z \zeta^\dagger 
$$
Notice that the second term in the right hand side is a rank one matrix.

The last equation is the operator ODE \eqref{geom_phase} as claimed.
\end{proof}

Now, in order to establish the relation between natural gradients and projective Kuramoto dynamics \eqref{Kuramoto_complex}-\eqref{coupling_complex}, compare \eqref{low_dim_Kuramoto_complex} with \eqref{nat_grad_flow_complex}.
It is evident that the coupling strength $K$ should be proportional to the optimization step $\kappa$, weights $w_j$ equal to $\gamma_j$ and initial oscillators' positions $z_j(0)$ chosen at evaluation points $z_j$. Then the dynamics \eqref{Kuramoto_complex}-\eqref{coupling_complex} implements natural evolution strategy on $\mathbb S^{2m-1}$.

\begin{rem}
Dynamics of the unitary matrix $R$ are described by \eqref{geom_phase}. In our setup, this ODE decouples from the ODE \eqref{low_dim_Kuramoto_complex} for $\zeta$. This decoupling is due to the specific form of the coupling function \eqref{coupling_complex}. For the general coupling function $Z$ the matrix $R$ explicitly enters ODE for $\zeta$. Such a situation arises, for instance, in the model where coupling coefficients $K$ are matrices. 

In any case, ODE for $R$ contains an important information as it measures accumulation of a non-Abelian geometric phase during the evolution. This effect of the unitary transformation $R$ unveils the link with quantum NES \cite{ADA-G,ZCSV,SIKC} and holonomic quantum computation \cite{ZR}. This link towards quantum decision-making will be discussed very briefly in the concluding section.
\end{rem}

\subsection{Algorithms}

Statements proven in the previous section show that the stochastic policies encoded by the Bergman family $sB(\zeta)$ introduce the effect of holonomy in computations. This holonomy entails an intrinsic uncertainty and makes the algorithm more subtle. 

One option is to act in analogy with Algorithm 1 explained in Section 4. This algorithm performs small natural gradient updates using Kuramoto model \eqref{Kuramoto_complex} on a small time interval. However, this implies an additional burden of computation of the holonomy by solving \eqref{geom_phase} at each update. 

One can circumvent computations of holonomy by estimating the maximum likelihood under the model $sB(\zeta)$. This would be analogue of Algorithm 2 from Section 4. In this case one considers probability measures on complex balls $\mathbb S^{2m-1}$ concentrated at $M$ atoms and compute holomorphic barycenters of these measures. \footnote{Holomorphic barycenter is an analogue of the conformal barycenter for Bergman balls and their isometry groups $SU(m,1)$.} As shown in \cite{JK,JK2}, these barycenters are computed by dynamics \eqref{Kuramoto_complex}-\eqref{coupling_complex} with repulsive coupling $K<0$. For the sake of accuracy, underline that \cite{JK,JK2} deal with barycenters of probability measures in balls, rather than on spheres. However, it is straightforward to adapt the analysis therein to probability measures on spheres.

\begin{algorithm}[H]
\caption{Maximum likelihood updates under the model $sB(\zeta)$ by computing stationary points of the IGO flow \eqref{nat_grad_flow_complex}}
\begin{algorithmic}[1]

\State Sample $M$ points $z_1,\dots,z_M$ from the uniform distribution on ${\mathbb S}^{2m-1}$.

\State Evaluate $f(z_1),\dots,f(z_M)$. Assign the weights $\gamma_1,\dots,\gamma_M$ to oscillators based on evaluations. To this end, use a certain suitable shaping scheme to ensure that $1/2>\gamma_i>0$ for all $i=1,\dots,M$. 

\State Denote by $\nu$ the probability measure on $\mathbb S^{2m-1}$ concentrated at $M$ points $z_1,\dots,z_M$ with the corresponding weights $\gamma_1,\dots,\gamma_M$. Find the holomorphic barycenter of the measure $\nu$. Denote it by $\zeta_\nu$.

\State Sample $M$ points $z_1,\dots,z_M$ from the probability distribution $sB(\zeta_\nu)$. Iterate steps 2) and 3).

\end{algorithmic}
\end{algorithm}

\section{Discussion and outlook}

In the present paper we have considered the problem which has a very simple statement: maximize the black-box function on a sphere. We presented a principled information-geometric approach to this problem, based on two families of probability distributions on spheres that, to our best knowledge, have not previously appeared in optimization or ML.

Despite the apparent problem simplicity, we believe that our results may have broader implications beyond the basic BBO setup. In particular, we designed IGO flows for our families and demonstrated an elegant connection with evolutions on isometry groups of two classes of hyperbolic balls. This further implies that Kuramoto models can serve as flexible computational tools for estimation and optimization on spheres. Invariance properties of the two families enable efficient and stable computations of gradients. Moreover, geometry of hyperbolic balls enable semantically rich deep feature representations and provides a powerful framework for deep RL and Bayesian inference in probabilistic graphical models.

In the very end we point out several outlooks for further advances.

\begin{enumerate}

\item [i)] {\bf Black-box optimization in ${\mathbb R}^N$ via directional stochastic search.}
Our findings provide a hint to one more principled approach to the classical black-box optimization in Euclidean spaces. Indeed, we propose directional natural evolution strategy, where the algorithm would sample directions from one of the distributions (spherical Cauchy or Bergman) and radii from a suitable distribution (for instance, exponential or Pareto) on the positive real line. The method for the natural gradient update of directions is presented here.

It would be interesting to examine if such an approach could compete with (versions of) CMA-ES on some benchmark BBO problems. This question merits a separate detailed study and is out of the scope of the present theoretical study.

\item[ii)] {\bf A framework for geometric RL.}
The second prospective line of research might be applications to more involved setups of RL with spherical or hyperbolic deep features. Indeed, natural policy gradients for the Gaussian family introduced in RL by Kakade \cite{Kakade} are widely used in some of the most prominent algorithms, including natural actor-critic \cite{PS}, episodic RL \cite{HN} and Trust Region Policy Optimization \cite{SLAJM}. We believe that our results can serve as a foundation for the design of principled IGO algorithms for directional and hierarchical deep features.

\item [iii)] {\bf Relation with quantum decision-making.}
Along our computations of Fisher information a principled difference between hyperbolic balls in real and complex vector spaces appeared. Indeed, it turns out that Poincar\' e balls (and, accordingly, the spherical Cauchy family) are isotropic, while the Bergman balls (and, accordingly, the Bergman family) are anisotropic (indeed, notice the term $\zeta \zeta^\dagger$ in \eqref{Fisher_Bergman}).

While the effect of holonomy appears in both models, the quantum geometric phase (non-Abelian Berry phase) arises in Bergman balls. Indeed, the ODE \eqref{geom_phase} for the unitary matrix $R$ includes the accumulation of this non-Abelian geometric phase during the update of strategies. This effect inevitably affect our algorithms. This can be seen as a manifestation of the quantum nature of the Bergman balls and Bergman stochastic policies. This observation paves the way towards quantum decision-making \cite{BB}, variational quantum computation \cite{ADA-G,ZR} and quantum NES \cite{ADA-G,ZCSV,SIKC}. The non-Abelian geometric phase arises as an inherent property of the quantum decision-making. This allows for encoding contextual ambiguity, such as the effect of emotional state on decision making. 

\end{enumerate}

Overall, we believe that the present study provides a rigorous framework for many diverse and exciting research directions.

\begin{appendices}

\section{Appendix A}

In this Appendix we compute the Fisher information for families $sC(a)$ and $sB(\zeta)$, respectively.

{\bf A1. Fisher information for the spherical Cauchy family: Proof of Proposition \ref{Fisher_info_Cauchy}}

\begin{proof}

The first argument is that M\" obius invariance of a family implies M\" obius invariance of the Fisher information for that family. Denote by $g_{ij}$ elements of the Fisher information and by $h$ a M\" obius transformation \eqref{Mobius_ball}.

Consider a family of probability distributions parametrized by a point $a \in \mathbb{B}^d$.

Let $\tilde a = h(a)$. Then, by the M\" obius invariance densities satisfy
$$
p(x;h^{-1}(\tilde a)) = p(h(x);\tilde a) \frac{d \sigma(h(x))}{d \sigma(x)}
$$
where $\sigma$ is a metric element on $\mathbb S^{d-1}$.

By taking logarithm and differentiating the above formula and taking into account that the last multiplier does not depend on $a$ we obtain the equality of scores:
$$
\partial_{\tilde a_i} \log p(x;h^{-1}(\tilde a)) = \partial_{\tilde a_i} \log p(h(x);\tilde a).
$$
Therefore, entries of the Fisher information matrix satisfy
$$
\tilde g_{ij}(\tilde a) \equiv \mathbb{E}_{\tilde a} [\partial_{\tilde a_i}\log p(x;h^{-1}(\tilde a)) \; \partial_{\tilde a_j} \log p(x;h^{-1}(\tilde a))] = \mathbb{E}_{\tilde a} [\partial_{\tilde a_i} \log p(x;\tilde a) \; \partial_{\tilde a_j} \log p(x;\tilde a)] \equiv g_{ij}(\tilde a).
$$

The second argument is the following: since the M\" obius group $Isom(\mathbb{B}^d)$ acts transitively on $\mathbb{B}_d$, any Riemannian metric on $\mathbb B^d$ invariant under the group $Isom(\mathbb{B}^d)$ must be proportional to the Poincar\' e metric. Hence, 
$$
g_{ij}(a) = c \frac{4 \delta_{ij}}{(1 - \| a\|^2)^2} \mbox{  for some constant } c. 
$$

Therefore, for the Fisher information metric for the family $sC(a)$ we have that $F_{sC}(a) = c \frac{4 (d-1)^2}{d} I$. It only remains to determine constant $c$. To that end, it suffices to evaluate $F_{sC}(0)$.

In order to do that, first write down the score:
$$
\partial_{a_i} \log p(x;a) = 2 (d-1) \left[ \frac{x_i - a_i}{\| x- a \|^2} - \frac{a_i}{1 -\| a \|^2} \right].
$$
Evaluating at zero:
$$
\partial_{a_i} \log p(x;a) \bigg|_{a=0} = 2 (d-1) x_i.
$$
Hence entries of $F_{sC}(0)$ are
$$
(F_{sC}){_{ij}}(0) = 4 (d-1)^2 \int_{\mathbb{S}^{d-1}} x_i x_j d \sigma(x).
$$
Using $\sum x_i^2 = 1$ we have that 
$$
\int_{\mathbb S^{d-1}} x_i x_j d \sigma(x) = \frac{\delta_{ij}}{d}.
$$
Substituting in the Fisher matrix we find that 
$$
(F_{sC}){_{ij}}(0) = 4 \frac{(d-1)^2}{d} \delta_{ij}.
$$
This shows that the constant $c$ equals $(d-1)^2/d$ and concludes the proof of the formula \eqref{Fisher_conf_nat}.
\end{proof}

{\bf A2. Fisher information for the Bergman family: Proof of Proposition \ref{Fisher_info_Bergman_prop}}

\begin{proof}

Let $p_{sB}(z;\zeta)$ be the density of the Bergman distribution given by \eqref{Bergman_dens}. Then the score reads
$$
\nabla_\zeta \log p(z;\zeta) = m \left( \frac{\zeta}{1 - \zeta^\dagger z} - \frac{z}{1 - \| z \|^2} \right).
$$
In particular, the score at zero has a very simple form
\begin{equation}
\label{aux1}
\nabla_\zeta \log p(0;\zeta) = m \zeta.
\end{equation}
Then it is straightforward to evaluate entries of the Fisher information at zero:
\begin{equation}
\label{aux2}
(F_{sB})_{ij}(0) = \int \limits_{{\mathbb S}^{2m-1}} \frac{\partial \log p(0;\zeta)}{\partial z_i}  \frac{\partial \log p(0;\zeta)}{\partial \bar z_j} p(0;\zeta) d \sigma(\zeta),
\end{equation}
where $\sigma(\zeta)$ is the volume element on the sphere.

Substituting $p(0,\zeta) = 1$ and \eqref{aux1} into \eqref{aux2} yields
$$
(F_{sB})_{ij}(0) = \int \limits_{{\mathbb S}^{2m-1}} (m \bar \zeta_i)(m \zeta_j) \cdot 1 \cdot d \sigma(\zeta) = m^2 \int \limits_{{\mathbb S}^{2m-1}} \bar \zeta_i \zeta_j d \sigma(\zeta).
$$
Now, using the spherical symmetry and holomorphic invariance, we have that
$$
\mbox{for} \; i \neq j \; \int \limits_{{\mathbb S}^{2m-1}} \bar \zeta_i \zeta_j d \sigma(\zeta) = 0; \quad \mbox{for} \; i = j \; \int \limits_{{\mathbb S}^{2m-1}} |\zeta_i|^2 d \sigma(\zeta) \quad \mbox{are equal.}
$$
Since $\sum_i |\zeta_i|^2 = \| \zeta \| = 1$ we have
$$
\int \limits_{{\mathbb S}^{2m-1}} \bar \zeta_i \zeta_j d \sigma(\zeta) = \frac{1}{m} \delta_{ij}
$$
where $\delta_{ij}$ is the Kronecker symbol.

This finally yields
$$
(F_{sB})_{ij}(0) = m^2 \frac{1}{m} \delta_{ij} = m \delta_{ij}.
$$
By applying the holomorphic mapping \eqref{Bergman_transf} to the last equation we find that
$$
(F_{sB})_{ij}(\zeta) = \frac{m}{(1 - \| \zeta \|^2)^2} \left[(1 - \| \zeta \|^2 \delta_{ij} + \zeta_i \bar \zeta_j ) \right].  
$$
The last equation yields the expression \eqref{Fisher_Bergman} for the Fisher information matrix $F_{sB}(\zeta)$.
\end{proof}

Hence, we ascertained the form of the Fisher matrix for the family of Bergman distributions on spheres. However, it is still non-trivial task to find the inverse of this matrix, as stated in Proposition \ref{inverse_Fisher_Bergman_prop}. Here, we prove this Proposition.

\begin{proof}

Denote 
$$
c = \frac{m}{1 - \| \zeta \|^2}, \quad v = \frac{\sqrt{m}}{1- \| \zeta \|^2} \zeta.
$$
Then the Fisher information matrix \eqref{Fisher_Bergman} can be written as $F_{sB}(\zeta) = c I + \zeta \zeta^\dagger$. 

Now, apply the Sherman-Morisson formula \cite{SM} for arbitrary complex matrix $A$ and complex vector $v$:  
$$
(A + v v^\dagger)^{-1} = A^{-1} - \frac{A^{-1}v v^\dagger A^{-1}}{1 + v^\dagger A^{-1} v}.
$$
In our case $A = c I$. Plugging the expression for the vector $v$ and making some simple rearrangements we get
$$
F_{sB}(\zeta)^{-1} = \frac{1}{c} I - \frac{1 -\| \zeta \|^2}{m} \zeta \zeta^\dagger.
$$ 
Finally, plugging the expression for $c$ into the last equations we get the formula \eqref{inverse_Fisher_Bergman} for the inverse of $F_{sB}$.
\end{proof}

\section{Appendix B}

In this very short Appendix we prove the auxiliary lemma used in the proof of the Proposition \ref{low_dim_Bergman}.

\begin{lemma}
Suppose that vector $v \in {\mathbb C}^m$ satisfies
\begin{equation}
\label{Moebius_lemma}
v = R \frac{\beta - w}{1 - w^\dagger \beta} \quad \mbox{  with } \; R \in U(d).
\end{equation}
Then
\begin{equation}
\label{identity1}
\frac{1}{1 - w^\dagger \beta} = \frac{1 + w^\dagger R^\dagger v}{1 - \| w\|^2} \quad \mbox{and} \quad 
\beta = \frac{R^\dagger v + w}{1 + w^\dagger R^\dagger v}.
\end{equation}
\end{lemma}

\begin{proof}
Multiply \eqref{Moebius_lemma} by $R^\dagger$ to get
\begin{equation}
\label{expr_lemma}
R^\dagger v (1 - w^\dagger \beta) = \beta - w \implies \beta = w + R^\dagger v (1 - w^\dagger \beta).
\end{equation}
Multiply the last equality by $w^\dagger$ to get
$$
w^\dagger \beta = \| w \|^2 + w^\dagger R^\dagger v (1 - w^\dagger \beta).
$$
Substitute the last equality from one to get
$$
1 - w^\dagger \beta = 1 - \|w\|^2 - w^\dagger R^\dagger v (1 - w^\dagger \beta) \implies (1 - w^\dagger \beta) (1 + w^\dagger R^\dagger v) = 1 - \|w\|^2.
$$
Rearranging multipliers in the last equality yields the first (left-side) expression in \eqref{identity1}.

In order to derive the second equality in \eqref{identity1} substitute the first equality into \eqref{Moebius_lemma}. This yields:
$$
v = R (\beta - w) \frac{1 + w^\dagger R^\dagger v}{1 - \| w \|^2}.
$$
Now, it is easy to solve the last expression w. r. to $\beta$ to get the right-hand equality in \eqref{identity1}.

\end{proof}

{\bf Acknowledgments}

\medskip

{\bf Funding support.}
The author acknowledges support from the Ministry of Education, Science and Innovation of Montenegro, projects "Mathematical analysis, optimization and machine learning" and "Complex-analytic and geometric techniques for non-Euclidean machine learning: theory and applications".

{\bf Use of AI technologies.}
The author used Gemini and Claude in order to refine mathematical arguments and computations on Fisher information (presented in Appendix A) and complex Kuramoto models (presented in subsection 6.1 and Appendix B). Discussions with Gemini and Claude also helped conceptualization of the present study. After using this service, the author reviewed and edited the content and takes full responsibility for the content of the publication.

\end{appendices}

\end{document}